\documentclass[12pt,final]{l4dc2023}

\usepackage{enumitem}
\usepackage{mathtools}
\usepackage{pgfplots}
\usepackage{tikz}
\usepackage{multicol}
\usepackage{booktabs} 

\usetikzlibrary{shapes,arrows}

\newtheorem{problem}{Problem}

\newcommand{\bbD}{\mathbb{D}}

\newcommand{\bbN}{\mathbb{N}}

\newcommand{\bbR}{\mathbb{R}}
\newcommand{\bbS}{\mathbb{S}}

\newcommand{\calC}{\mathcal{C}}

\newcommand{\calF}{\mathcal{F}}

\newcommand{\calI}{\mathcal{I}}

\newcommand{\calL}{\mathcal{L}}

\newcommand{\calP}{\mathcal{P}}

\DeclareMathOperator{\diag}{diag}

\DeclareMathOperator{\CNN}{CNN}
\DeclareMathOperator{\Pool}{Pool}
\DeclareMathOperator{\blkdiag}{blkdiag}

\definecolor{mycolor1}{rgb}{0.00000,0.44700,0.74100}%
\definecolor{mycolor2}{rgb}{0.85000,0.32500,0.09800}%
\definecolor{mycolor3}{rgb}{0.92900,0.69400,0.12500}%
\definecolor{mycolor4}{rgb}{0.49400,0.18400,0.55600}%

\tikzstyle{block} = [rectangle, draw, fill=blue!20, 
    text width=20em, text centered, rounded corners, minimum height=2em]
\tikzstyle{block2} = [rectangle, draw, fill=red!20, 
    text width=20em, text centered, rounded corners, minimum height=2em]
\tikzstyle{block3} = [rectangle, draw, fill=yellow!20, 
    text width=20em, text centered, rounded corners, minimum height=2em]
\tikzstyle{block4} = [text width=25em, text centered, minimum height=2em]
\tikzstyle{line} = [draw, -latex']

\title[Lipschitz constant estimation for 1D convolutional neural networks]{Lipschitz constant estimation for 1D convolutional neural networks}
\usepackage{times}

\author{%
\Name{Patricia Pauli} \Email{patricia.pauli@ist.uni-stuttgart.de}\\
 \addr University of Stuttgart, Institute for Systems Theory and Automatic Control, 70569 Stuttgart, Germany%
 \AND
\Name{Dennis Gramlich} \Email{dennis.gramlich@ic.rwth-aachen.de}\\
 \addr RWTH Aachen, Chair of Intelligent Control Systems, 52074 Aachen, Germany%
\AND
\Name{Frank Allgöwer} \Email{frank.allgower@ist.uni-stuttgart.de}\\
 \addr University of Stuttgart, Institute for Systems Theory and Automatic Control, 70569 Stuttgart, Germany%
}

\begin{document}

\maketitle

\begin{abstract}%
    In this work, we propose a dissipativity-based method for Lipschitz constant estimation of 1D convolutional neural networks (CNNs). In particular, we analyze the dissipativity properties of convolutional, pooling, and fully connected layers making use of incremental quadratic constraints for nonlinear activation functions and pooling operations. The Lipschitz constant of the concatenation of these mappings is then estimated by solving a semidefinite program which we derive from dissipativity theory. To make our method as efficient as possible, we exploit the structure of convolutional layers by realizing these finite impulse response filters as causal dynamical systems in state space and carrying out the dissipativity analysis for the state space realizations. The examples we provide show that our Lipschitz bounds are advantageous in terms of accuracy and scalability.
\end{abstract}


\section{Introduction}\label{sec:intro}
Convolutional neural networks (CNNs) achieve excellent results in practical applications, wherein convolutional layers detect meaningful features in small sections of an input signal while necessitating significantly fewer parameters than fully connected layers. CNNs have hence become the state of the art in many machine learning applications. While 2D CNNs are prevalently used in image and video processing, applications of 1D CNNs include classification of medical data, health monitoring, fault detection in electrical machines, and audio processing \citep{kiranyaz20211d,oord2016wavenet}. Safety-critical applications especially require neural networks (NNs) to perform robustly and reliably, for which the Lipschitz constant has become a generally accepted robustness measure \citep{szegedy2013intriguing}. Thus, recently, efforts have been made to finding accurate upper bounds on the Lipschitz constant of NNs. For example, \citet{latorre2020lipschitz} formulated a polynomial optimization problem for Lipschitz constant estimation and in particular, they study Lipschitz bounds of general CNNs with respect to the $\ell_\infty$ norm. \citet{combettes2020lipschitz} proposed bounds viewing activation functions as averaged operators and \citet{fazlyab2019efficient} derived a semidefinite program (SDP) based on incremental quadratic constraints (QCs) to over-approximate the nonlinear activation functions. However, both latter methods have limited scalability to NNs of practically relevant scale. One reason for this limitation in \citep{fazlyab2019efficient} is the use of a sparse linear matrix inequality (LMI). \citet{newton2021exploiting,xue2022chordal} suggested to exploit the underlying structure of this LMI, i.\,e., its chordal sparsity pattern, 
to break down the corresponding LMI into multiple smaller LMIs. Similarly, we consider properties of individual layers, also yielding a set of smaller layer-wise LMIs instead of one large and sparse one, which is advantageous in terms of computational tractability of the underyling SDP. 

Previous SDP-based approaches \citep{fazlyab2019efficient,xue2022chordal} are formulated for fully connected NNs, yet they can be applied to CNNs by transforming the convolutional layers to fully connected layers. This corresponds to studying the Toeplitz matrices of the convolutional layers \citep{pauli2022neural,aquino2022robustness}, that, unfortunately, are highly redundant, i.\,e., they have a high degree of sparsity and repeated entries, which causes a significant computational overhead. Our approach of viewing convolutional layers as dynamical systems reduces this redundancy by exploiting the structure of CNNs in SDP-based Lipschitz constant estimation, which thus leads to better scalability. This observation that convolutional layers are dynamical systems is non-trivial to exploit, since other components of CNNs, such as pooling layers, are not linear time-invariant systems. We handle this heterogeneous feed-forward interconnection of systems (layers) by carrying out a dissipativity analysis for each layer type separately. Subsequently, we derive the Lipschitz bound of the input-output mapping based on the dissipativity properties of each individual layer and the underlying feed-forward interconnection of the CNN, like it is oftentimes done in multi-agent control \citep{arcak2016networks} or has been suggested for the robustness analysis of NNs \citep{aquino2022robustness}. Other than previous works, our approach considers pooling layers in the Lipschitz constant estimation, by deriving incremental QCs for them. In this work, we focus on 1D CNNs to lay out our novel dissipativity-based concept for the derivation of an SDP for Lipschitz constant estimation. However, using a compact 2D systems respresenation of 2D CNNs \citep{gramlich2022convolutional}, the framework can be extended to the more popular class of 2D CNNs. Note further that in the same way \citep{fazlyab2019efficient} can be applied to CNNs, our approach can be applied to 2D CNNs, necessitating to represent 2D convolutions as 1D convolutions.

Our main contributions can be summarized as follows. We provide a first result on how to exploit CNN structures in SDP-based Lipschitz constant estimation. To this end, we introduce a dissipativity-based method to derive layer-wise LMIs. Both our compact description of convolutions and the layer-wise LMIs improve the computational tractability over previous approaches. In Section \ref{sec:problem_statement}, we formally state the problem and in Section \ref{sec:dissipativity}, we outline the dissipativity-based approach for Lipschitz constant estimation for 1D CNNs. In doing so, we introduce the description of convolutional layers in state space, derive incremental QCs for the nonlinear activation functions and pooling layers, and subsequently, we establish an SDP to determine an upper bound on the Lipschitz constant for a given 1D CNN. Finally, in Section~\ref{sec:sim} we compare our approach to \cite{fazlyab2019efficient} in terms of computation time and accuracy and, in Section \ref{sec:conclusion}, we conclude the paper.

\textbf{Notation}: By $\bbS^n$ ($\bbS_+^n$), we denote the set of $n$-dimensional symmetric (positive definite) matrices. By $\bbD^n$ ($\bbD_+^n$), we denote the set of $n$-dimensional (positive definite) diagonal matrices, and by $\bbN_+$ the natural numbers without zero. $\calI$ is a set of indices with elements $i\in\bbN_+$, and $\vert\calI\vert$ gives the number of elements in the index set $\calI$. The subscripts $a$, $b$ in $\nu_a$ and $\nu_b$ label two instances of $\nu$.

\section{Problem statement}\label{sec:problem_statement}
We consider a 1D CNN as shown in Fig.~\ref{fig:1D_CNN}, consisting of convolutional layers $\calC_i: \bbR^{c_{i-1} \times N_{i-1}} \to \bbR^{c_i \times N_i}$ with indices $i\in\calI_C$, pooling layers $\calP_i: \bbR^{c_{i-1} \times N_{i-1}} \to \bbR^{c_i \times N_i}$ with indices $i\in\calI_P$, and fully connected layers $\calL_i: \bbR^{n_{i-1}} \to \bbR^{n_{i}}$ with indices $i\in\calI_F$, adding up to a total number of $l=\vert\calI_C\vert+\vert\calI_P\vert+\vert\calI_F\vert$ layers. 
Here, $c_{i-1}$ and $c_i$ denote the input and output channel sizes, $N_{i-1}$ and $N_i$ the input and output dimensions in pooling and covolutional layers, and $n_{i-1}$ and $n_{i}$ the input and output dimensions in fully connected layers, respectively. We study the CNN as a concatenation of the individual layers
\begin{align}\label{eq:CNN}
    \mathrm{CNN}_\theta = \calL_{l} \circ \ldots \circ \calL_{p+1} \circ \calF \circ \calP_{p} \circ \calC_{p - 1} \circ \ldots \circ \calP_{2} \circ \calC_1,
\end{align}
wherein the only formal restriction is the separation into two parts: (i) the part containing fully connected layers $\calI_F=\{p+1,\dots,l\}$ 
and (ii) the part consisting of convolutional and pooling layers $\calI_C\cup\calI_P=\{1,\dots,p\}$ 
where $p$ is the index of the last layer of part~(ii), which can either be a convolutional or a pooling layer. Further, the transition between the two parts requires a flattening operation $\calF:\bbR^{c_p\times N_p}\to\bbR^{n_p}$ of the output of the $p$-th layer, where $n_p=c_pN_p$. 
\begin{figure}
    \centering
    \includegraphics[width=\textwidth]{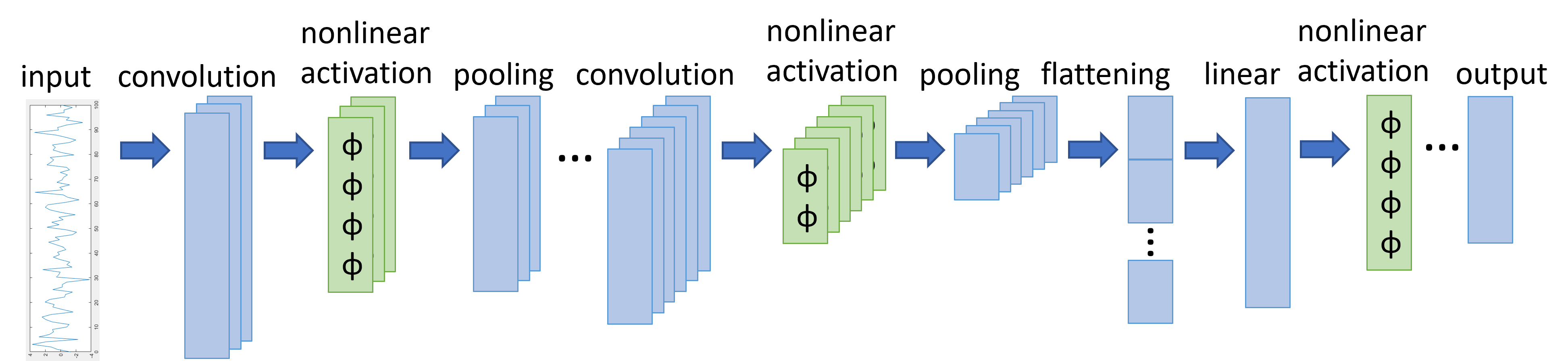}
    \caption{1D CNN structure of \eqref{eq:CNN} including convolutions, pooling and linear layers as well as nonlinear activation functions and a flattening layer.}
    \label{fig:1D_CNN}
\end{figure}

To enable an efficient description of the layers constituting the CNN \eqref{eq:CNN}, we denote the signals between those layers by $w^i$, i.\,e., $w^i = \calC_i (w^{i-1})$, $w^i = \calL_i (w^{i-1})$, or $w^i = \calP_i (w^{i-1})$ depending on the type of the $i$-th layer. Note that a signal $w^i$ can either be an element of $\bbR^{c_i \times N_i}$ in the case of the convolutional part (ii) of the network, or an element of $\bbR^{n_i}$ in the case of the fully connected part~(i) of the network. With this notation, we define each layer ($\calC_i$/$\calP_i$/$\calL_i$) through the way it acts on its input $w^{i-1}$. A \emph{convolutional} layer
\begin{equation}\label{eq:CNN_layer}
    \calC_i: w_k^{i} = \phi_i\left(b_i + \sum_{j=0}^{\ell_i-1} K_j^i w_{k-j}^{i-1}\right), \quad k = 0,\ldots,N_i-1,~\forall i\in\calI_C,
\end{equation}
with convolution kernel $K^i_j\in\bbR^{c_{i} \times c_{i-1}}$, $j=0,\dots,\ell_i-1$, kernel size $\ell_i$, and bias $b_i\in\bbR^{c_{i}}$, as the first instance, performs a convolution on $w^{i-1}\in\bbR^{c_{i-1}\times N_{i-1}}$ and subsequently, applies an element-wise nonlinear activation function $\phi_i: \bbR^{c_i} \to \bbR^{c_i}$ to obtain the output $w^{i}\in\bbR^{c_{i}\times N_{i}}$, wherein we denote all entries along the channel dimension at propagation/time step $k$ by $w_k^{i}\in\bbR^{c_i}$. 
Note that in \eqref{eq:CNN_layer} we define $w_{k-j}^{i-1}$ as zero whenever $k-j \leq 0$, i.\,e., we apply zero padding.

In order to downsample a signal $w^{i-1}$ convolutional layers are potentially followed by \emph{pooling} layers. To this end, we consider average pooling layers and maximum pooling layers
\begin{align}\label{eq:pooling_layer}
    \calP_i^\mathrm{av} : w_k^{i} = \frac{1}{\ell_i}\sum_{j = 1}^{\ell_i} w^{i-1}_{\ell_i (k - 1) + j}, \quad
    \calP_i^\mathrm{max} : w_k^i =\max_{j = 1,\ldots, \ell_i} w^{i-1}_{\ell_i (k - 1) + j} ,\quad
    \begin{split}
        k = 0,\ldots,N_i-1,\\
        \forall i\in\calI_P,
    \end{split}
\end{align}
where $\ell_i$ is the size of the pooling kernel, which both are standard components in CNNs. Throughout this paper, we only consider the case that stride and kernel size are both $\ell_i$, such that the output dimension is $N_i=N_{i-1}/\ell_i$, which requires that $\ell_i$ divides $N_{i-1}$. However, the approach can be extended to the general case of stride and pooling kernel size being different.

Finally, a CNN typically holds \emph{fully connected} layers, which we define as mappings
\begin{equation}\label{eq:FC_layer}
    \calL_i: w^i=\phi_i(W_{i}w^{i-1}+b_{i})\quad \forall i\in\calI_F\backslash \{l\},\quad \calL_{l}: w^{l} = W_{l}w^{l-1}+b_{l}
\end{equation}
with weights $W_i\in\bbR^{n_{i}\times n_{i-1}}$, biases $b_i\in\bbR^{n_{i}}$ and activation functions $\phi_i: \bbR^{n_i} \to \bbR^{n_i}$. 

The CNN $f_\theta(w^0)=w^{l}$ is thus characterized by $\theta=\{\{K^i,b_i\}_{i\in\calI_C},\{W_i,b_i\}_{i\in\calI_F}\}$ and the chosen pooling layers and nonlinear activation functions. The aim of this paper is to find a Lipschitz certificate $\gamma$ for 1D CNNs.

\begin{problem}
    \label{problem1}
    For a given 1D convolutional neural network $\CNN_\theta$ with fixed parameter $\theta$, find an upper bound on the Lipschitz constant, i.\,e., find a value $\gamma > 0$ such that
    \begin{align}\label{eq:Lipschitz}
        \| \CNN_\theta (w_a^0) - \CNN_\theta (w_b^0) \|_2 \leq \gamma \| w_a^0 - w_b^0 \|_2\qquad \forall w_a^0,w_b^0\in\bbR^{N_0}.
    \end{align}
\end{problem}

\section{Dissipativity-based Lipschitz constant estimation for 1D CNNs}\label{sec:dissipativity}
To solve Problem \ref{problem1}, we establish dissipativity properties for the individual layers, i.\,e., we find 
$Q_i$, $S_i$, $R_i$ that satisfy the incremental quadratic constraints 
\begin{align}\label{eq:iQC}
    s(\Delta w^{i-1},\Delta w^{i})\coloneqq
    \begin{bmatrix}
        \underline{w}_a^i - \underline{w}_b^i\\
        \underline{w}_a^{i-1} - \underline{w}_b^{i-1}
    \end{bmatrix}^\top
    \begin{bmatrix}
        Q_i & S_i\\
        S_i^\top & R_i
    \end{bmatrix}
    \begin{bmatrix}
        \underline{w}_a^i - \underline{w}_b^i\\
        \underline{w}_a^{i-1} - \underline{w}_b^{i-1}
    \end{bmatrix} \geq 0~~~\forall \underline{w}^{i-1}_a,\underline{w}^{i-1}_b\in\bbR^{n_{i-1}}
\end{align}
for each $\calC_i$-, $\calL_i$-, and $\calP_i$-layer, where $\underline{w}^{i-1}$ denotes a column-wise stacked vector of dimension $n_i=c_i N_i$ if it comes from $\bbR^{c_i\times N_i}$ and corresponds to $w^{i-1}$ if it comes from $\bbR^{n_i}$, and where $\Delta w^i = w_a^i-w_b^i$. To further estimate the Lipschitz constant of the input-output mapping of the 1D CNN \eqref{eq:CNN}, we connect the different layers through their dissipativity properties, considering the CNN's cascaded interconnection structure. To account for this structure, we introduce the block-tridiagonal coupling matrix \citep{kanellopoulos2019decentralized}
\begin{equation}\label{eq:diss_coupling}
    H\coloneqq\begin{bmatrix}
        \gamma^2 I-R_1  & -S_1^\top       & 0 & \hdots & 0\\
        -S_1  & -Q_1-R_2  & -S_2^\top &\ddots  & \vdots\\
        0          & -S_2 & -Q_2-R_3   & \ddots & 0\\
        \vdots          & \ddots         & \ddots     & \ddots & -S_{l}^\top\\
        0 & \hdots & 0 & -S_{l} & -Q_{l}-I
    \end{bmatrix}\succeq 0,
\end{equation}
where the layer indices $\{1,\dots,l\}=\calI_C\cup\calI_P\cup\calI_F$ count through all $\vert\calI_C\vert$  convolutional layers, $\vert\calI_P\vert$ pooling layers, and $\vert\calI_F\vert$ fully connected layers. 
 
\begin{theorem}\label{thm:main}
    Let $\CNN_\theta$ and $\gamma > 0$ be given. If there exist
    $Q_i\in\bbS^{n_i}$, $S_i\in\bbR^{n_{i}\times n_{i-1}}$, $R_i\in\bbS^{n_{i-1}}$ such that \eqref{eq:iQC} is satisfied for all $i=1,\dots,l$ layers and if in addition \eqref{eq:diss_coupling} holds, then the $\CNN_\theta$ is $\gamma$-Lipschitz continuous.
\end{theorem}
\begin{proof}
    We left and right multiply \eqref{eq:diss_coupling} by
    $\begin{bmatrix}
        \Delta {\underline{w}^0}^\top & \Delta {\underline{w}^1}^\top & \dots & \Delta {\underline{w}^{l}}^\top
    \end{bmatrix}$ and its transpose, respectively, and obtain
\begin{align*}
    \gamma^2 {\Delta \underline{w}^0}^\top \Delta \underline{w}^0 \underbrace{ - s(\Delta w^0,\Delta w^1)}_{\leq 0} - \dots \underbrace{- s(\Delta w^{l-1},\Delta w^{l})}_{\leq 0} -{\Delta \underline{w}^{l}}^\top \Delta \underline{w}^{l} \geq 0,
\end{align*}
The supply functions $s(\Delta w^{i-1},\Delta w^i)$ are nonnegative by \eqref{eq:iQC} which implies $\gamma^2\Vert \Delta{w^0}\Vert_2^2 -\Vert \Delta w^{l}\Vert_2^2\geq 0$, i.\,e., $\gamma$-Lipschitz continuity of $\CNN_\theta$.
\end{proof}

In the following Sections \ref{sec:iQCs_CNN} to \ref{sec:iQCs_pooling}, we discuss incremental QCs for the different layer types and in Section \ref{sec:SDP}, we establish an SDP that renders Theorem \ref{thm:main} computational, providing an accurate upper bound on the Lipschitz constant of a given 1D CNN.

\subsection{Incremental quadratic constraints for convolutional layers}\label{sec:iQCs_CNN}
Previous approaches to analyze the Lipschitz constant of CNNs using LMIs and incremental QCs rely on rewriting convolutional layers as sparse fully connected layers using Toeplitz matrices \citep{pauli2022neural,aquino2022robustness}. In the present section, we introduce a compact and non-sparse state space representation for convolutional layers which significantly reduces the size of the convolutional layer description, yielding a scalable SDP-based approach for Lipschitz constant estimation. 

A discrete-time state space representation of the finite impulse response (FIR) filter \eqref{eq:CNN_layer} is, for example, given by
\begin{align*}
    x_{k+1}^i &= A_i x_k^i + B_i w^{i-1}_k,\quad y_k^i = C_i x^i_k + D_i w^{i-1}_k + b_i, \quad w_k^i=\phi(y_k^i) \quad \forall i\in\calI_C
\end{align*}
with state $x_k\in\bbR^{n_{x_i}}$ and state dimension $n_{x_i}=(\ell_i-1)c_{i-1}$. Here, the matrices $A_i$, $B_i$, $C_i$, $D_i$ are
\begin{equation*}
    A_i = \begin{bmatrix}
        0 & I\\
        0 & 0
    \end{bmatrix},~
    B_i = \begin{bmatrix}
        0 \\
        I 
    \end{bmatrix},~
    C_i =
    \begin{bmatrix}
        K^i_{\ell_i-1} & \dots & K^i_1
    \end{bmatrix},~
    D_i = K^i_0
\end{equation*}
or similarity transformations $(EAE^{-1},EB,CE^{-1},D)$ thereof for some invertible $E\in\bbR^{n_{x_i}\times n_{x_i}}$. Note that the formulation of a state space representation of \eqref{eq:CNN_layer} requires causality. In standard CNN literature, 1D convolutions are given by an acausal FIR filter $y_k^i = b_i + \sum_{j=-\tilde{\ell}_i}^{\tilde{\ell}_i} K_j^i w_{k-j}^{i-1}$. A change of indices, however, yields \eqref{eq:CNN_layer}. One clear advantage of representing the convolutional layer compactly in state space rather than by a Toeplitz matrix is that it is independent of the signal's input dimension $N_{i-1}$ and scales with the channel sizes $c_{i-1}$ and $c_{i}$ and the kernel size $\ell_i$ of the convolution filter instead. This means that the computational expense of analyzing a convolutional layer is the same for inputs of arbitrary input dimension $N_{i-1}$.

In addition to the convolution operation, a convolutional layer \eqref{eq:CNN_layer} consists of a nonlinear activation function. Commonly used activation functions, such as ReLU, $\tanh$ and sigmoid, are slope-restricted which can be captured by an incremental QC as follows.
\begin{lemma}[\cite{fazlyab2019efficient,pauli2021training}]\label{lem:slope}
    Let $\varphi:\bbR\to\bbR$ be slope-restricted on $[0,1]$. Then, for any $\Lambda_i\in\bbD_+^{c_i}$ the vector-valued function $\phi(\nu^i)=
    \begin{bmatrix}
        \varphi(\nu_1^i) & \dots & \varphi(\nu_{c_i}^i)    
    \end{bmatrix}^\top:\bbR^{c_i}\to\bbR^{c_i}$ satisfies
    \begin{equation}\label{eq:slope_restriction}
        \begin{bmatrix}
            \phi(\nu^i_a)-\phi(\nu^i_b)\\
            \nu^i_a-\nu^i_b
        \end{bmatrix}^\top
        \begin{bmatrix}
            -2\Lambda_i & \Lambda_i\\
            \Lambda_i & 0 
        \end{bmatrix}
        \begin{bmatrix}
            \phi(\nu^i_a)-\phi(\nu^i_b)\\
            \nu^i_a-\nu^i_b
        \end{bmatrix}\geq 0\quad\forall~\nu^i_a,\nu^i_b\in\bbR^{c_i}.
    \end{equation}
\end{lemma}

By virtue of Lemma \ref{lem:slope}, we replace constraints $w_a^i - w_b^i = \phi (\nu_a^i) - \phi (\nu_b^i)$ with the inequality \eqref{eq:slope_restriction} which relaxes the nonlinear constraint. This kind of approach is common practice in robust control \citep{scherer2000linear}, where control engineers routinely interpret nonlinearities, such as slope-restricted activation functions, as uncertainties and use QCs to perform a convex relaxation of the constraints imposed by these nonlinearities to render a problem computationally tractable.

Using these QCs and the state space realization of the convolution, we introduce the concept of layer-wise incremental dissipativity for an individual convolutional layer \eqref{eq:CNN_layer}. We denote the incremental difference of the unrolled signal by $\Delta w^i =w^i_a-w^i_b$ and the incremental difference at propagation step $k$ by $\Delta_w^i =w^i_{a,k}-w^i_{b,k}$.
\begin{definition}[Incremental dissipativity of convolutional layers]
    We call the $i$-th convolutional \linebreak layer $\calC_i$ incrementally $(Q,S,R)$-dissipative if any input-output pair $\{\Delta w^{i-1},\Delta w^{i}\}$ stemming from \eqref{eq:CNN_layer} satisfies
    \begin{equation}\label{eq:diss_CNN}
        s(\Delta w^{i-1},\Delta w^{i}) =  
        \sum_{k=0}^{N_i-1} {\Delta_w^{i}}^\top \widetilde{Q}_i \Delta_w^{i} + {2\Delta_w^{i}}^\top \widetilde{S}_i \Delta_w^{i-1} + {\Delta_w^{i-1}}^\top \widetilde{R}_i \Delta_w^{i-1} \geq 0 \quad \forall~N_i\in\bbN_+
    \end{equation}    
    with some matrices $\widetilde{Q}_i\in\bbS^{c_i}$,  $\widetilde{S}_i\in\bbR^{c_{i}\times c_{i-1}}$, $\widetilde{R}_i\in\bbS^{c_{i-1}}$.
\end{definition}
Such incremental dissipativity certificates can capture, e.g., Lipschitz continuity ($\widetilde{Q}_i=-I$, $\widetilde{R}_i=\gamma^2 I$, $\widetilde{S}_i=0$), slope restriction ($\widetilde{Q}_i=0$, $\widetilde{R}_i=2\Lambda_i$, $\widetilde{S}_i=-\Lambda_i$) or other behavioral properties such as incremental passivity ($\widetilde{Q}_i=\widetilde{R}_i=0$, $\widetilde{S}_i=I$) \citep{hill1976stability}. Our goal is to find some $\widetilde{Q}_i$, $\widetilde{S}_i$, and $\widetilde{R}_i$ such that the $i$-th convolutional layer satisfies \eqref{eq:diss_CNN}. The following lemma provides us with a sufficient condition for dissipativity of convolutional layers $\calC_i$.
\begin{lemma}\label{lem:diss_CNN}
    If for some matrices $\widetilde{Q}_i\in\bbS^{c_i}$,
    $\widetilde{S}_i\in\bbR^{c_{i}\times c_{i-1}}$, $\widetilde{R}_i\in\bbS^{c_{i-1}}$ there exist $P_i\in\bbS_+^{n_{x_i}}$, $\Lambda_i\in \bbD^{c_i}_+$ such that
    \begin{equation}\label{eq:cert_CNN}
        \begin{bmatrix}
            P_i-A_i^\top P_iA_i  & -A_i^\top P_iB_i & -C_i^\top\Lambda_i\\
            -B_i^\top P_iA_i &  \widetilde{R}_i-B_i^\top P_iB_i  & \widetilde{S}_i^\top-D_i^\top\Lambda_i\\
            -\Lambda_i C_i & \widetilde{S}_i-\Lambda_i D_i &2\Lambda_i+\widetilde{Q}_i
        \end{bmatrix}\succeq 0
    \end{equation}
    holds, then the $i$-th convolutional layer $\calC_i$ is $(Q,S,R)$-dissipative.
\end{lemma}
The proof follows along the lines of typical arguments used in dissipativity analysis \citep{goodwin2014adaptive} and robust control \citep{scherer2000linear} and is for completeness included in Appendix 1.

A convolutional layer that satisfies \eqref{eq:diss_CNN} also satisfies \eqref{eq:iQC},
choosing the parameterization
\begin{align*}
    Q_i = \blkdiag(\widetilde{Q}_i,\dots,\widetilde{Q}_i),~S_i =\blkdiag(\widetilde{S}_i,\dots,\widetilde{S}_i),~
    R_i =\blkdiag(\widetilde{R}_i,\dots,\widetilde{R}_i),
\end{align*}
where the matrices $\widetilde{Q}_i$, $\widetilde{S}_i$, $\widetilde{R}_i$ appear $N_i$ times as entries of block-diagonal matrices, which means that \eqref{eq:cert_CNN} ensures the incremental QC \eqref{eq:iQC} for convolutional layers $\calC_i$. 

\subsection{Incremental quadratic constraints for fully connected layers}\label{sec:iQCs_FC}
Similarly, we next consider fully connected layers \eqref{eq:FC_layer} of the CNN individually and address incremental dissipativity for these layers.

\begin{definition}[Incremental dissipativity of fully connected layers]
    We call the $i$-th fully connected layer $\calL_i$ incrementally $(Q,S,R)$-dissipative if any input-output pair $\{\Delta w^{i-1}, \Delta w^{i}\}$ stemming from \eqref{eq:FC_layer} satisfies
    \begin{equation}\label{eq:diss_FC}
        s(\Delta w^{i-1},\Delta w^{i}) =  
        {\Delta w^{i}}^\top Q_i \Delta w^{i} +
        2{\Delta w^{i}}^\top S_i \Delta w^{i-1} + {\Delta w^{i-1}}^\top R_i \Delta w^{i-1}\geq 0
    \end{equation}    
    with some matrices $Q_i\in\bbS^{n_i}$, $S_i\in\bbR^{n_{i}\times n_{i-1}}$,  $R_i\in\bbS^{n_{i-1}}$.
\end{definition}
\begin{lemma}\label{lem:diss_FC}
    If for some matrices $Q_i\in\bbS^{n_i}$, $S_i\in\bbR^{n_{i}\times n_{i-1}}$, $R_i\in\bbS^{n_{i-1}}$ there exist $\Lambda_i\in\bbD_+^{n_i}$ such that
    \begin{equation}\label{eq:cert_FC}
        \begin{bmatrix}
            R_{i} & S_i^\top-W_{i}^\top \Lambda_{i}\\
            S_i-\Lambda_{i} W_{i} & 2\Lambda_{i}+Q_{i}
        \end{bmatrix}\succeq 0
    \end{equation}
    holds, then the $i$-th fully connected layer $\calL_i$ is $(Q,S,R)$-dissipative.
\end{lemma}
The proof follows as a special case of the proof of Lemma \ref{lem:diss_CNN} with $A_i=0$, $B_i=0$, $C_i=0$, and $D_i=W_i$, as well as $c_i=n_i$, $N_i=1$ and, accordingly, the LMI \eqref{eq:cert_FC} ensures $(Q,S,R)$-dissipativity \eqref{eq:diss_FC} / the incremental QC \eqref{eq:iQC} for fully connected layers \eqref{eq:FC_layer}.

\subsection{Incremental quadratic constraints for pooling layers}\label{sec:iQCs_pooling}
Besides the relaxation of activation functions \citep{fazlyab2019efficient}, we newly introduce incremental QCs to handle maximum and average pooling layers. These pooling layers downsample a signal channel-wise, i.\,e., multiple successive time/propagation steps of the input signal are mapped to one time/propagation step of the output signal. Such operations cannot easily be captured in a state space formulation which is why we represent them via incremental QCs instead. Alternatively, linear pooling layers, such as average pooling, can be unrolled to fully-connected layers, again yielding a sparse and redundant formulation thereof \citep{fazlyab2019efficient,pauli2022neural}. The use of incremental QCs suggested in this paper is not only more efficient but also allows to account for nonlinear maximum pooling layers, which is not possible in \citep{fazlyab2019efficient}.

\emph{Average pooling} $\calP_i^{\mathrm{av}}$ as defined in \eqref{eq:pooling_layer} is applied channel-wise such that for the $j$-th channel, we can define
\begin{equation*}
    w_{jk}^{i} = \frac{1}{\ell_i}\sum_{s = 1}^{\ell_i} w^{i-1}_{j,\ell_i (k - 1) +s} \eqqcolon\Pool^\mathrm{av}(v^{i}_{jk})
    , \quad
        k = 0,\ldots,N_i-1,
\end{equation*}
where $v_{jk}^i = \begin{bmatrix}
    w^{i-1}_{j,\ell_i (k - 1)+1} & \dots & w^{i-1}_{j,\ell_i k} 
\end{bmatrix}^\top\in\bbR^{\ell_i}$.

\begin{proposition}\label{prop:Lip_av_pooling}
The average pooling operation $\Pool^{\mathrm{av}}$ is Lipschitz continuous with $\mu_i=\frac{1}{\sqrt{\ell_i}}$, i.\,e.,
\begin{equation}\label{eq:Lip_av_pool}
    \vert w_{a,jk}^i-w_{b,jk}^i\vert^2\leq\mu_i^2\Vert v^{i}_{a,jk}-v^{i}_{b,jk}\Vert^2 \quad j = 1,\dots,c_i,~k=1,\dots,N_i\quad\forall v^{i}_{a,jk}, v^{i}_{b,jk}\in\bbR^{\ell_i}.
\end{equation}
\end{proposition}
The proof can be found in Appendix 2. While \eqref{eq:Lip_av_pool} only considers Lipschitz continuity of one channel, the following incremental QC holds for all channels, for which we define the vertically flattened vector 
\begin{equation*}
\underline{v}^i_k\!=\!
\begin{bmatrix}
    w^{i-1}_{1,{\ell_i}(k-1)+1}\! & \!\dots\! & \! w^{i-1}_{c_i,{\ell_i}(k-1)+1}\! & \!w^{i-1}_{1,{\ell_i}(k-1)+2}\! & \!\dots\! & \! w^{i-1}_{c_i,{\ell_i}(k-1)+2} \! & \!\dots\! & \! w^{i-1}_{1,{\ell_i} k} \!  & \!\dots\! & \! w^{i-1}_{c_i,{\ell_i} k}
\end{bmatrix}^\top,
\end{equation*}
that collects entries from all $c_i=c_{i-1}$ channels and from all $\ell_i$ propagation steps (of the $(i-1)$-th signal) that are combined in the pooling operation corresponding to propagation step $k$ (of the $i$-th signal).

\begin{lemma}\label{lem:av_pooling}
    For all $T_i\in\bbS_+^{c_i}$, the average pooling operation $\calP_i^{\mathrm{av}}$ fulfills
    \begin{align}\label{eq:lemma_av_pool}
        \begin{bmatrix}
            w^i_{a,k}-w^i_{b,k}\\
            \underline{v}^{i}_{a,k}-\underline{v}^{i}_{b,k}
        \end{bmatrix}^\top
        \left[
        \begin{array}{c|ccc}
            -T_i     & 0 & \hdots & 0\\\hline
            0      & \mu_i^2T_i & & \\
            \vdots & & \ddots & 0\\
            0      & & 0 & \mu_i^2T_i
        \end{array}
        \right]
        \begin{bmatrix}
            w^i_{a,k}-w^i_{b,k}\\
            \underline{v}^{i}_{a,k}-\underline{v}^{i}_{b,k}
        \end{bmatrix}
        \geq 0, \quad
        \begin{split}
            k = 1,\dots,N_i, \\ \forall \underline{v}^{i}_{a,k}, \underline{v}^{i}_{b,k}\in\bbR^{\ell_i c_i}.
        \end{split}
    \end{align}
This implies that all $i\in\calI_P^{\mathrm{av}}$ average pooling layers are $(Q,S,R)$-dissipative, i.\,e., they satisfy \eqref{eq:iQC} with $Q_i=-\blkdiag(T_i,\dots,T_i)\in\bbS^{n_{i}}$, $S_i = 0$, $R_i=\mu_i^2\blkdiag(T_i,\dots,T_i)\in\bbS^{n_{i-1}}$.
\end{lemma}

\begin{proof}
The positive semidefinite matrix $T_i$ has a factorization $T_i=L_i^\top L_i$, $L_i\in\bbR^{c_i\times c_i}$, which we use to define $\underline{\tilde{v}}_k^i\coloneqq\blkdiag(L_i,\dots,L_i)\underline{v}_k^i$, $k=1,\dots,N_i$. 
As average pooling is a linear operator, $L_i$ can be pulled out of the pooling operation, i.\,e., $\tilde{w}_k^i=L_iw_k^i$. 
Hence, \eqref{eq:lemma_av_pool} is equivalent to
    \begin{equation*}
        \begin{bmatrix}
            \tilde{w}^i_{a,k}-\tilde{w}^i_{b,k}\\
            \underline{\tilde{v}}_{a,k}^{i}-\underline{\tilde{v}}_{b,k}^{i}
        \end{bmatrix}^\top
        \left[
        \begin{array}{cc}
            -I_{c_i} & 0\\
            0 & \mu_i^2I_{\ell_i {c_i}} 
        \end{array}
        \right]
        \begin{bmatrix}
            \tilde{w}^i_{a,k}-\tilde{w}^i_{b,k}\\
            \underline{\tilde{v}}^{i}_{a,k}-\underline{\tilde{v}}^{i}_{b,k}
        \end{bmatrix}
        \geq 0, \quad k = 1,\dots,N_i,\quad \forall \underline{\tilde{v}}_{a,k}^{i}, \underline{\tilde{v}}_{b,k}^{i}\in\bbR^{\ell_i c_i},
    \end{equation*}
which is the channel-wise stacked up version of the Lipschitz condition \eqref{eq:Lip_av_pool}. This can easily be seen by reordering the entries of the flattened vectors $\underline{\tilde{v}}^{i}_{a,k}$ and $\underline{\tilde{v}}^{i}_{b,k}$. Stacking up the incremental QC \eqref{eq:lemma_av_pool} for all $k=1,\dots,N_i$ propagation steps yields a special case of \eqref{eq:iQC}.
\end{proof}

An alternative commonly used pooling operation is \emph{maximum pooling} $\calP_i^\mathrm{max}$ as defined in \eqref{eq:pooling_layer}. This nonlinear operation is $1$-Lipschitz in case the kernel size and the stride coincide.
\begin{lemma}\label{lem:max_pooling}
    For all $\Sigma_i\in\bbD_+^{c_i}$, the maximum pooling operation $\calP_i^\mathrm{\max}$ fulfills
    \begin{align}\label{eq:lemma_max_pool}
        \begin{bmatrix}
            w^i_{a,k}-w^i_{b,k}\\
            \underline{v}^{i}_{a,k}-\underline{v}^{i}_{b,k}
        \end{bmatrix}^\top
        \left[
        \begin{array}{c|ccc}
            -\Sigma_i & 0 & \hdots & 0 \\\hline
            0& \Sigma_i &&\\
            \vdots& & \ddots & \\
            0& & & \Sigma_i
        \end{array}
        \right]
        \begin{bmatrix}
            w^i_{a,k}-w^i_{b,k}\\
            \underline{v}^{i}_{a,k}-\underline{v}^{i}_{b,k}
        \end{bmatrix}
        \geq 0,\quad
        \begin{split}
            \quad k = 1,\dots,N_i,\\
            \forall \underline{v}^{i}_{a,k}, \underline{v}^{i}_{b,k}\in\bbR^{\ell_i c_i}.
        \end{split}
    \end{align}
This implies that all $i\in\calI_P^{\mathrm{max}}$ maximum pooling layers are $(Q,S,R)$-dissipative, i.\,e., they satisfy \eqref{eq:iQC} with $Q_i=-\blkdiag(\Sigma_i,\dots,\Sigma_i)\in\bbD^{n_{i}}$, $S_i = 0$, $R_i=\blkdiag(\Sigma_i,\dots,\Sigma_i)\in\bbD^{n_{i-1}}$.
\end{lemma}
\begin{proof}
    Condition \eqref{eq:lemma_max_pool} is the stacked up version of the incremental QCs
\begin{equation*}
    \sigma_i\vert w^i_{a,jk}-w^i_{b,jk}\vert^2\leq\sigma_i\Vert v^{i}_{a,jk}-v^{i}_{b,jk}\Vert^2, \quad j = 1,\dots,c_i,~k = 1,\dots,N_i,\quad\forall v^{i}_{a,jk}, v_{b,jk}^{i}\in\bbR^{\ell_i},
\end{equation*}
which holds true for all nonnegative scalars $\sigma_i\geq0$ by $1$-Lipschitz continuity of the maximum pooling operation. The multiplier is then $\Sigma_i = \diag(\sigma_1,\dots,\sigma_{c_i})$. Accordingly to the proof of  Lemma \ref{lem:av_pooling}, we can stack up \eqref{eq:lemma_max_pool} for all $k=1,\dots,N_i$ propagation steps to obtain a special case of \eqref{eq:iQC}.
\end{proof}

The incremental QC \eqref{eq:lemma_av_pool} describing average pooling allows full matrices $T_i$ as multipliers, whereas maximum pooling requires $\Sigma_i$ in \eqref{eq:lemma_max_pool} to be diagonal. This result corresponds to the main results in \citep{kulkarni2002incremental} who address that only linear operations maintain incremental positivity. Note that \citet{fazlyab2019efficient} wrongly proposed full matrices $\Lambda_i$ as multipliers in \eqref{eq:slope_restriction} which was later corrected by \citet{pauli2021training}.

\subsection{Semidefinite program for Lipschitz constant estimation}\label{sec:SDP}
In this section, we propose an SDP to certify an upper bound on the Lipschitz constant for a given 1D CNN, that renders Theorem \ref{thm:main} computational.

Pooling layers are incrementally $(Q,S,R)$-dissipative by design according to Lemma \ref{lem:av_pooling} and Lemma \ref{lem:max_pooling}, respectively, and in addition Lemma~\ref{lem:diss_CNN} and  Lemma~\ref{lem:diss_FC} provide us LMIs to enforce dissipativity onto convolutional and fully connected layers such that all $\calC_i$-, $\calP_i$-, and $\calF_i$-layers in \eqref{eq:CNN} are incrementally $(Q,S,R)$-dissipative. Our goal is to determine an accurate Lipschitz bound of the input-output mapping defined by a given 1D CNN, for which the dissipativity properties of the individual layers are a utility. Hence, in all $i\in\calI_C\cup\calI_F$, we consider $Q_i$, $S_i$, and $R_i$ as decision variables that provide us degrees of freedom in the optimization. 
To reduce the number of decision variables, which is computationally favorable, we parameterize $Q_i$, $S_i$, and $R_i$ in such a way that $H$ as defined in \eqref{eq:diss_coupling} is equal to zero. This yields $R_1=Q_0=\gamma^2 I$, $S_i=0$, and $R_i=-Q_{i-1}$, $i=1,\dots,l$, $Q_l=-I$. Further, we note that in pooling layers $Q_i$ and $R_i$ are both parameterized by the same matrix, such that we use this information to specify the structure of $\widetilde{Q}_{i-1}$ of the preceding convolutional layer, which yields Corollary~\ref{cor:reduced_vars}.

\begin{corollary}\label{cor:reduced_vars}
    Let $\CNN_\theta$ and $\gamma > 0$ be given and let all activation functions be slope-restricted on $[0,1]$. Further, we choose  $\widetilde{Q}_0=-\tilde{\gamma}^2 I$, where $\gamma = \tilde{\gamma} \prod_{s\in\calI_{P}^\mathrm{av}} \mu_{s}$. If there exist
    \begin{enumerate}[label=(\roman*)]
        \item  $\widetilde{Q}_i\in\bbS^{c_i}$ ($\widetilde{Q}_i\in\bbD^{c_i}$ if a convolutional layer is followed by a maximum pooling layer), $P_i\in\bbS_+^{n_{x_i}}$, and $\Lambda_i\in\bbD_+^{c_i}$ such that
        \begin{align}\label{eq:cert_reduced_vars_CNN}
            \begin{split}
            \begin{bmatrix}
            P_i-A_i^\top P_iA_i  & -A_i^\top P_iB_i & -C_i^\top\Lambda_i\\
            -B_i^\top P_iA_i &  -\widetilde{Q}_{i-1}-B_i^\top P_iB_i  & -D_i^\top\Lambda_i\\
            -\Lambda_i C_i & -\Lambda_i D_i &2\Lambda_i+\widetilde{Q}_i
            \end{bmatrix}\succeq 0\quad \forall i\in\calI_C,
            \end{split}
        \end{align}
        \item $Q_i\in\bbS^{n_i}$ and $\Lambda_i\in\bbD_+^{n_i}$ such that
\begin{align}\label{eq:cert_reduced_vars_FC}
    \begin{split}
    \begin{bmatrix}
        -Q_{i-1} & -W_{i}^\top \Lambda_{i}\\
        -\Lambda_{i} W_{i} & 2\Lambda_{i}+Q_{i}
    \end{bmatrix}\succeq 0\quad \forall i = \calI_F \backslash\{l\}, \quad    
    \begin{bmatrix}
        -Q_{l-1} & -W_{l}^\top\\
        -W_{l} & I
    \end{bmatrix}\succeq 0,
    \end{split}
\end{align}
    \end{enumerate}
    then the $\CNN_\theta$ is $\gamma$-Lipschitz continuous.
\end{corollary}

A complete proof is given in Appendix 3. For a given CNN, we can find an upper bound on the Lipschitz constant by solving the semidefinite program
\begin{equation}\label{eq:SDP}
    \min_{\gamma^2,Q,\Lambda,P}\gamma^2 \quad\text{s.\,t.}\quad \eqref{eq:cert_reduced_vars_CNN},\eqref{eq:cert_reduced_vars_FC},
\end{equation}
where $\Lambda=\{\Lambda_i\}_{i\in\calI_C\cup \calI_F\backslash \{l\}}$, $Q=\{Q_i\}_{i\in\calI_C\cup \calI_F\backslash \{l\}}$, $P=\{P_i\}_{i\in\calI_C}$.

\begin{remark}
The matrices $Q_i$ carry information on the layer-wise worst-case gain and worst-case direction of the corresponding input to the layer. In case a maximum pooling layer exists, we lose the information of the worst-case direction through the parameterization of $Q_i$ as a diagonal matrix.
\end{remark}

\begin{remark}
    For fully connected NNs, the set of layer-wise LMIs \eqref{eq:cert_reduced_vars_FC} is equivalent to the LMI in \citet{fazlyab2019efficient} by chordal sparsity \citep{xue2022chordal}. For CNNs our approach is slightly less accurate than \cite{fazlyab2019efficient} , but unlike Fazlyab, we can provide Lipschitz constant estimates that hold for all input dimensions. In addition, our method utilizes multipliers $\Lambda_i$ of dimension $c_i$, whereas using Toeplitz matrices and the approach in \cite{fazlyab2019efficient} requires the use of larger multipliers of size $c_iN_i$ for each layer, i.e., more degrees of freedom. 
\end{remark}

\section{Simulation results}\label{sec:sim}
\begin{figure}
\begin{minipage}{0.3\textwidth}
%
%
\definecolor{mycolor1}{rgb}{0.00000,0.44700,0.74100}%
\definecolor{mycolor2}{rgb}{0.85000,0.32500,0.09800}%
\definecolor{mycolor3}{rgb}{0.92900,0.69400,0.12500}%
\definecolor{mycolor4}{rgb}{0.49400,0.18400,0.55600}%
\begin{tikzpicture}

\begin{axis}[%
width=1.2in,
height=0.65in,
at={(0.677in,2.149in)},
scale only axis,
xmin=3,
xmax=60,
xlabel style={font=\color{white!15!black}},
ymin=60,
ymax=200,
ylabel style={font=\color{white!15!black}},
ylabel={Lip. bound},
axis background/.style={fill=white},
legend style={at={(0.97,0.03)}, anchor=south east, legend cell align=left, align=left, draw=white!15!black},
label style={font=\footnotesize},
tick label style={font=\tiny},
legend style={font=\tiny}
]
\addplot [color=mycolor1, line width=1pt]
  table[row sep=crcr]{%
3	140.306592595461\\
6	140.306592595461\\
9	140.306592595461\\
12	140.306592595461\\
15	140.306592595461\\
18	140.306592595461\\
21	140.306592595461\\
24	140.306592595461\\
27	140.306592595461\\
30	140.306592595461\\
33	140.306592595461\\
36	140.306592595461\\
39	140.306592595461\\
42	140.306592595461\\
45	140.306592595461\\
48	140.306592595461\\
51	140.306592595461\\
54	140.306592595461\\
57	140.306592595461\\
60	140.306592595461\\
};

\addplot [color=mycolor2, line width=1pt]
  table[row sep=crcr]{%
3	64.6011917257882\\
6	107.832760715923\\
9	122.184387503777\\
12	127.980360591263\\
15	130.822291195024\\
18	132.414695545216\\
21	133.393971693215\\
24	134.038535146783\\
27	134.485325649141\\
30	134.807755253411\\
33	135.048142618723\\
36	135.232181112527\\
39	135.37627464757\\
42	135.491200776381\\
45	135.584361858702\\
48	135.660936888243\\
51	135.724650176949\\
54	135.77823167359\\
57	135.823727358751\\
60	135.862697283693\\
};

\addplot [color=mycolor3, line width=1pt]
  table[row sep=crcr]{%
3	83.3391522848191\\
6	120.93281241941\\
9	133.388850837922\\
12	138.501086709154\\
15	141.038344591818\\
18	142.470773356887\\
21	143.355706988429\\
24	143.939696107219\\
27	144.344944557659\\
30	144.637485192993\\
33	144.855492616225\\
36	145.022266137605\\
39	145.15267281041\\
42	145.256558002194\\
45	145.34064985826\\
48	145.409672141268\\
51	145.467019820294\\
54	145.515183501405\\
57	145.55602362819\\
60	145.590952548754\\
};

\addplot [color=mycolor4, line width=1pt]
  table[row sep=crcr]{%
3	127.988995958013\\
6	167.379786514624\\
9	179.589424827064\\
12	184.654268215384\\
15	187.208286457326\\
18	188.669054813512\\
21	189.580680413561\\
24	190.187081422862\\
27	190.610559667602\\
30	190.917841591238\\
33	191.147814495091\\
36	191.324373680137\\
39	191.462854275179\\
42	191.573461658552\\
45	191.663199786488\\
48	191.737004309172\\
51	191.798434157485\\
54	191.850107740155\\
57	191.893986076025\\
60	191.931561286121\\
};

\end{axis}

\begin{axis}[%
width=1.2in,
height=0.5in,
at={(0.677in,1.488in)},
scale only axis,
xmin=3,
xmax=60,
xlabel style={font=\color{white!15!black}},
xlabel={input dimension $N_0$},
ymode=log,
ymin=0.0001,
ymax=235.2623588,
yminorticks=true,
ylabel style={font=\color{white!15!black}},
ylabel={Comp. time},
axis background/.style={fill=white},
label style={font=\footnotesize},
tick label style={font=\tiny},
legend style={font=\tiny}
]
\addplot [color=mycolor1, line width=1pt, forget plot]
  table[row sep=crcr]{%
3	0.3203857\\
60	0.3203857\\
};
\addplot [color=mycolor2, line width=1pt, forget plot]
  table[row sep=crcr]{%
3	1.0518046\\
6	2.6372734\\
9	1.2694363\\
12	2.2486946\\
15	3.7465999\\
18	5.6314204\\
21	9.2214331\\
24	13.5990067\\
27	21.9202642\\
30	30.3662525\\
33	42.4561613\\
36	53.1794025\\
39	65.6260397\\
42	114.3584952\\
45	132.9288481\\
48	133.9066401\\
51	146.9539268\\
54	189.1498867\\
57	235.2623588\\
60	231.0006264\\
};
\addplot [color=mycolor3, line width=1pt, forget plot]
  table[row sep=crcr]{%
3	0.4961973\\
6	0.3472111\\
9	0.5425934\\
12	0.9012752\\
15	1.4863074\\
18	2.2858194\\
21	3.5611259\\
24	5.9025435\\
27	11.0999751\\
30	13.2803821\\
33	16.0834352\\
36	23.0618194\\
39	25.5367588\\
42	62.5058672\\
45	41.6215871\\
48	60.2990043\\
51	71.111382\\
54	84.9275442\\
57	101.6343933\\
60	107.574662\\
};
\addplot [color=mycolor4, line width=1pt, forget plot]
  table[row sep=crcr]{%
3	0.0003373\\
6	0.0006666\\
9	0.000777\\
12	0.0009471\\
15	0.0028044\\
18	0.0017485\\
21	0.0029518\\
24	0.0025063\\
27	0.003842\\
30	0.0042227\\
33	0.0069212\\
36	0.0055628\\
39	0.009929\\
42	0.0125265\\
45	0.0125091\\
48	0.0133787\\
51	0.0157536\\
54	0.0121658\\
57	0.0199408\\
60	0.0185558\\
};
\end{axis}
\end{tikzpicture}%
\end{minipage}
\begin{minipage}{0.34\textwidth}
        \centering
    {\tiny

    \begin{tikzpicture}[node distance = 0.65cm, auto]
        \node [block4] (input) {
        Input signal: 128 (dimension) x 1 (channel)
        };
        \node [block, below of=input] (CNN1) {
        Conv: 3 kernel + front padding: 128 x $c_1$
        };
        \node [block2, below of=CNN1,node distance = 0.7cm] (avpool1) {
        Average pool: 2 kernel + 2 stride: 128 x $c_1$
        };
        \node [block, below of=avpool1] (CNN2) {
        Conv: 3 kernel + front padding: 64 x $c_2$
        };
        \node [block2, below of=CNN2,node distance = 0.7cm] (avpool2) {
        Average pool: 2 kernel + 2 stride: 32 x $c_2$
        };
        \node [block3, below of=avpool2,node distance = 0.7cm] (FC) {
        Dense: fully connected layer
        };
        \node [block4, below of=FC] (output) {
        Output: 1 of 5 classes
        };
        \path [line] (input) -- (CNN1);
        \path [line] (CNN1) -- node [midway,align=right] {ReLU} (avpool1);
        \path [line] (avpool1) -- (CNN2);
        \path [line] (CNN2) -- node [midway,align=right] {ReLU} (avpool2);
        \path [line] (avpool2) -- node [midway,align=right] {flatten} (FC);
        \path [line] (FC) -- (output);
    \end{tikzpicture}}
\end{minipage}
\begin{minipage}{0.33\textwidth}
%
%
\definecolor{mycolor1}{rgb}{0.00000,0.44700,0.74100}%
\definecolor{mycolor2}{rgb}{0.85000,0.32500,0.09800}%
\definecolor{mycolor3}{rgb}{0.92900,0.69400,0.12500}%
\definecolor{mycolor4}{rgb}{0.49400,0.18400,0.55600}%
\begin{tikzpicture}

\begin{axis}[%
width=1.48in,
height=0.5in,
at={(0.677in,1.5in)},
scale only axis,
bar shift auto,
xmin=0.509090909090909,
xmax=6.49090909090909,
xtick={1,2,3,4,5,6},
xticklabels={{[2,4]},{[4,8]},{[6,12]},{[8,16]},{[10,20]},{[12,24]}},
ymin=0,
ymax=2,
ylabel style={font=\color{white!15!black}},
ylabel near ticks,
ylabel={Lip. bound},
ybar=1pt,
axis background/.style={fill=white},
label style={font=\footnotesize},
tick label style={font=\tiny},
legend style={font=\tiny}
]
\addplot[ybar, bar width=2, fill=mycolor1, draw=black, area legend] table[row sep=crcr] {%
1	1\\
2	1\\
3	1\\
4	1\\
5	1\\
6	1\\
};
\addplot[ybar, bar width=2, fill=mycolor2, draw=black, area legend] table[row sep=crcr] {%
1	0.748898090940656\\
2	0.603426214879626\\
3	0.588040660804116\\
4	0.569223734406961\\
5	0.593130115502298\\
6	0.725289749474632\\
};
\addplot[ybar, bar width=2, fill=mycolor3, draw=black, area legend] table[row sep=crcr] {%
1	0.891901230161797\\
2	0.792112352737788\\
3	0.703233713769858\\
4	0.711509274997895\\
5	0.707419348121777\\
6	0.907130729936554\\
};
\addplot[ybar, bar width=2, fill=mycolor4, draw=black, area legend] table[row sep=crcr] {%
1	1.2486717565769\\
2	1.46523719443985\\
3	1.53266573791949\\
4	1.93538545673751\\
5	1.69877515674121\\
6	1.83525472664121\\
};
\end{axis}

\begin{axis}[%
width=1.48in,
height=0.5in,
at={(0.677in,0.673in)},
scale only axis,
bar shift auto,
xmin=0.509090909090909,
xmax=6.49090909090909,
xtick={1,2,3,4,5,6},
xticklabels={{[2,4]},{[4,8]},{[6,12]},{[8,16]},{[10,20]},{[12,24]}},
xlabel style={font=\color{white!15!black}},
xlabel={channel sizes $[c_1,c_2]$},
ymin=0,
ymax=120,
ylabel style={font=\color{white!15!black}},
ylabel={Comp. time},
ylabel near ticks,
ybar=1pt,
axis background/.style={fill=white},
label style={font=\footnotesize},
tick label style={font=\tiny},
legend style={font=\tiny}
]
\addplot[ybar, bar width=2, fill=mycolor1, draw=black, area legend] table[row sep=crcr] {%
1	1\\
2	1\\
3	1\\
4	1\\
5	1\\
6	1\\
};

\addplot[ybar, bar width=2, fill=mycolor2, draw=black, area legend] table[row sep=crcr] {%
1	5.89813364806204\\
2	50.3075333729295\\
3	107.678406668425\\
4	86.674077110748\\
5	88.7999724220606\\
6	95.3374607831382\\
};

\addplot[ybar, bar width=2, fill=mycolor3, draw=black, area legend] table[row sep=crcr] {%
1	4.7622499453954\\
2	38.8299575227987\\
3	73.3793232765399\\
4	59.7661190831457\\
5	69.8534112206139\\
6	77.0466399453686\\
};

\addplot[ybar, bar width=2, fill=mycolor4, draw=black, area legend] table[row sep=crcr] {%
1	0.0279138802954191\\
2	0.0124685387477629\\
3	0.0139361796356678\\
4	0.0120715924130467\\
5	0.00520358937254051\\
6	0.0055707473364852\\
};

\end{axis}
\end{tikzpicture}%
\end{minipage}
\vspace{-0.3cm}
\caption{Comparison of our method (\textbf{\textcolor{mycolor1}{--}})
, LipSDP-Neuron (\textbf{\textcolor{mycolor2}{--}}), LipSDP-Layer (\textbf{\textcolor{mycolor3}{--}}), and the spectral norm product (\textbf{\textcolor{mycolor4}{--}}).
Left: Lipschitz bounds and computation time for fully convolutional neural network over input dimension, 
middle: architecture of CNN, right: Normalized Lipschitz bounds and normalized computation times over increasing channel sizes.}
\label{fig:conv}
\vspace{-0.7cm}
\end{figure}
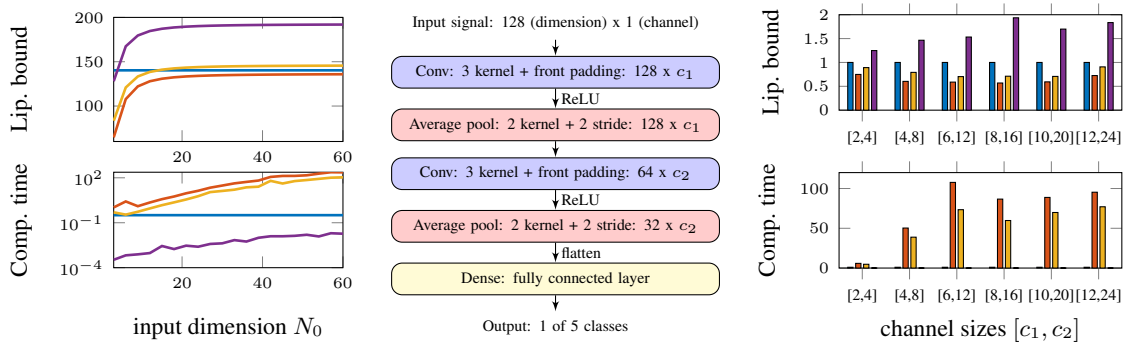

In the following, we compare our approach for Lipschitz constant estimation to the approaches LipSDP-Neuron and LipSDP-Layer suggested by \cite{fazlyab2019efficient} and to the trivial bound obtained from multiplication of spectral norms of weight matrices \citep{szegedy2013intriguing}
\subsection{Fully convolutional neural networks}
We set up a fully convolutional neural network with three convolutional layers, no pooling layers, and no fully connected layers with channel dimensions $[1,3,5,10]$ and kernel sizes $[3,3,3]$ and randomly generate a set of convolution kernels for the CNN. In Fig.~\ref{fig:conv} on the left, we see the Lipschitz bounds for different input signal dimensions $N_0$ and the corresponding computation times. We observe that for small input dimensions the Lipschitz bounds LipSDP-Neuron and LipSDP-Layer are superior to our approach. However, for larger input dimensions the computation time increases drastically and, in this example, our method outperforms LipSDP-Layer for input dimensions $N_0>15$.

\subsection{Deep CNNs}
For the CNN shown in Fig.~\ref{fig:conv} in the middle, we generate six different CNNs with different channel sizes $c_1$ and $c_2$ and compare the Lipschitz bounds and the computation times of our method, LipSDP-Neuron, LipSDP-Layer, and the spectral norm product. While the computational overhead of LipSDP-Neuron/Layer is huge the Lipschitz bound is comparable to the one obtained using our approach which is clearly more accurate than the spectral norm bound. We note that, unlike LipSDP-Neuron/-Layer, our approach scales to significantly larger CNNs than shown in Fig. \ref{fig:conv}.

\section{Conclusion and future work}\label{sec:conclusion}
In this paper, we proposed an efficient SDP to determine accurate upper bounds on the Lipschitz constant for 1D CNNs based on dissipativity theory. To do so, we derived incremental QCs for pooling layers and activation functions and further chose a compact state space representation for convolutional layers. A future direction is to directly exploit the structure of a 2D convolution for improved scalability.

\acks{The authors thank Ruigang Wang and Ian Manchester for helpful discussions on the topic. This work was funded by Deutsche Forschungsgemeinschaft (DFG, German Research Foundation) under Germany's Excellence Strategy - EXC 2075 - 390740016 and under grant 468094890. We acknowledge the support by the Stuttgart Center for Simulation Science (SimTech). The authors thank the International Max Planck Research School for Intelligent Systems (IMPRS-IS) for supporting Patricia Pauli and Dennis Gramlich.}

\bibliography{references}

\section*{Appendix}

\subsection*{Appendix 1: Proof of Lemma \ref{lem:diss_CNN}}\label{app:1}
\begin{proof}
We left/right multiply \eqref{eq:cert_CNN} with
$\begin{bmatrix}
    {\Delta_x^i}^\top & {\Delta_w^{i-1}}^\top & {\Delta_w^i}^\top   
\end{bmatrix}$ and its transpose, respectively, yielding
    \begin{align*}
        {\Delta_x^i}^\top P_i \Delta_x^i - {\Delta_+^i}^\top P_i \Delta_+^i 
        +{\Delta_w^i}^\top \widetilde{Q}_i\Delta_w^i
        +2{\Delta_w^{i-1}}^\top \widetilde{S}_i \Delta_w^i
        +{\Delta_w^{i-1}}^\top \widetilde{R}_i{\Delta_w^{i-1}}
        \\
        \underbrace{-2{\Delta_w^i}^\top\Lambda_i \Delta_y^i+2{\Delta_w^i}^\top\Lambda_i \Delta_w^i}_{\leq 0}
        \geq 0,
    \end{align*}
where we introduced $\Delta_x^i = x_{a,{k}}^i-x_{b,{k}}^i$ and $\Delta_+^i = x_{a,{k+1}}^i-x_{b,{k+1}}^i$, respectively. Using Lemma~\ref{lem:slope}, we can drop the term that is associated with slope restriction of the activation function as indicated. We next sum up from $k=0$ to $N_i-1$ and obtain
    \begin{align*}
        \left(x^i_{a,0}-x^i_{b,0}\right)^\top P_i \left(x^i_{a,0}-x^i_{b,0}\right) \underbrace{-\left(x^i_{a,{N_i}}-x^i_{b,{N_i}}\right)^\top P_i \left(x^i_{a,{N_i}}-x^i_{b,{N_i}}\right)}_{\leq0}
        +s(\Delta w^{i-1},\Delta w^i)\geq 0.
    \end{align*}
As $P_i$ is positive definite, we can drop the second term. Furthermore, zero padding at the beginning of the signal corresponds to setting $x_0=0$ such that the first term can also be dropped, finally yielding \eqref{eq:diss_CNN}.
\end{proof}

\subsection*{Appendix 2: Proof of Proposition \ref{prop:Lip_av_pooling}}\label{app:2}
\begin{proof}
Without loss of generality, we consider the case $k=1$ in this proof. We rewrite both sides of the inequality in terms of the scalar differences $\Delta w_{jl}^{i-1}=w_{a,jl}^{i-1}-w_{b,jl}^{i-1}$:
\begin{align*}
    \left\vert w_{a,j1}^i-w_{b,j1}^i\right\vert^2=\left\vert\frac{1}{\ell_i}\sum_{l=1}^{\ell_i} w_{a,jl}^{i-1}-\frac{1}{\ell_i}\sum_{l=1}^{\ell_i} w_{b,jl}^{i-1}\right\vert^2
    =\frac{1}{\ell_i^2}\left(\Delta w_{j1}^{i-1}+\dots+\Delta w_{j\ell_i}^{i-1}\right)^2\\
    =\frac{1}{\ell_i^2}\left(\left(\Delta {w_{j1}^{i-1}}\right)^2+\dots+\left(\Delta {w_{j\ell_i}^{i-1}}\right)^2+2\Delta w_{j1}^{i-1}\Delta w_{j2}^{i-1}+\dots+2\Delta w_{j{\ell_i-1}}^{i-1}\Delta w_{j\ell_i}^{i-1}\right)\\
    \leq \frac{1}{\ell_i}\left(\left(\Delta {w_{j1}^{i-1}}\right)^2+\dots+\left(\Delta {w_{j\ell_i}^{i-1}}\right)^2\right)=\frac{1}{\ell_i}\left\Vert v^{i}_{a,j1}-v^{i}_{b,j1}\right\Vert^2.
\end{align*}
Rearranging the resulting inequality results in
\begin{align*}
    0 \leq
    {\Delta v^{i}_{j1}}^\top
    \begin{bmatrix}
        \frac{1}{\ell_i}-\frac{1}{\ell_i^2} & -\frac{1}{\ell_i^2} & \dots & -\frac{1}{\ell_i^2} \\
        -\frac{1}{\ell_i^2} & \frac{1}{\ell_i}-\frac{1}{\ell_i^2} & \ddots & \vdots \\
        \vdots & \ddots &  \ddots & -\frac{1}{\ell_i^2}\\
        -\frac{1}{\ell_i^2} & \dots & -\frac{1}{\ell_i^2} & \frac{1}{\ell_i}-\frac{1}{\ell_i^2}
    \end{bmatrix} \Delta v^{i}_{j1}=\Delta{v^{i}_{j1}}^\top M \Delta v_{j1}^{i},
\end{align*}
wherein the matrix $M$ is diagonally dominant for all $\ell_i\in\bbN_+$ and consequently a positive semidefinite matrix.
\end{proof}

\subsection*{Appendix 3: Proof of Corollary \ref{cor:reduced_vars}}\label{app:3}
\begin{proof}
We show the combination of Theorem \ref{thm:main} using Lemma \ref{lem:diss_CNN}, Lemma \ref{lem:diss_FC}, Lemma \ref{lem:av_pooling}, and Lemma \ref{lem:max_pooling}  together with the parameterization obtained from setting $H=0$ in \eqref{eq:diss_coupling} yields Corollary \ref{cor:reduced_vars}.
First, we note that \eqref{eq:diss_coupling} holds for $Q_i\in\bbS^{n_{i}}$, $R_i\in\bbS^{n_{i-1}}$, $S_i\in\bbS^{n_{i-1}\times n_{i}}$ replaced by $\widetilde{Q}_i\in\bbS^{c_{i-1}}$, $\widetilde{R}_i\in\bbS^{c_{i-1}}$, $\widetilde{S}_i\in\bbS^{c_{i-1}\times c_{i}}$ for all $i=1,\dots,p-1$ convolutional and pooling layers, in addition leaving $Q_p$ as is and replacing $S_p$ by $\begin{bmatrix}\widetilde{S}_p & \dots & \widetilde{S}_p\end{bmatrix}$ and replacing $R_p$ by $\widetilde{R}_p$ in the $p$-th layer to account for the flattening operation.

Next, we insert $\widetilde{Q}_i=-T_i$, $\widetilde{S}_i=0$, $\widetilde{R}_i=\mu_i^2T_i$ in all $i\in\calI_P^\mathrm{av}$ average pooling layers, yielding
\begin{equation}\label{eq:diss_coupling2}
\begin{bmatrix}
        \ddots & \ddots\\
        \ddots & -\widetilde{Q}_{i-2}-R_{i-1} & -\widetilde{S}_{i-1}^\top &&&&\\
        &-\widetilde{S}_{i-1} & -\widetilde{Q}_{i-1}-\mu_{i}^2T_{i}  &&&&\\
        &&& T_{i}-\widetilde{R}_{i+1} & -\widetilde{S}_{i+1}^\top &\\
        &&& -\widetilde{S}_{i+1} & \widetilde{Q}_{i+1}-\widetilde{R}_{i+2} & \ddots\\
        &&&& \ddots & \ddots
    \end{bmatrix}\succeq 0.
\end{equation}
We then left and right multiply the resulting matrix \eqref{eq:diss_coupling2} with $\blkdiag(1/\mu_i^2 I_{m^i_1}, I_{m^i_2})$, where $m^i_1={\sum_{j=0}^i n_j}$ and $m^i_2=\sum_{j=i+1}^{l} n_j$. For layers $j=1,\dots,i-1$, we point out that $s(\Delta w^{j-1}, \Delta w^j)\geq 0$ holds by \eqref{eq:cert_CNN}. Instead, we can verify $\mu_{i}^2 s(w^{i-1},w^i)\geq 0$ as $\mu_{i}^2\geq 0$, for which we redefine $\widetilde{Q}_i\coloneqq \mu_{i}^2 \widetilde{Q}_i$, $\widetilde{S}_i\coloneqq \mu_{i}^2 \widetilde{S}_i$, $\widetilde{R}_i\coloneqq \mu_{i}^2 \widetilde{R}_i$ by their rescaled versions. We perform this step for all average pooling layers. The first term in the upper left block then is $\tilde{\gamma}^2 I$ with $\tilde{\gamma}^2 = \gamma^2 / \prod_{s\in\calI_{P}^\text{av}} \mu_{s}^2$.

Next, we insert $\widetilde{Q}_i=-\Sigma_i$, $\widetilde{S}_i=0$, $\widetilde{R}_i=\Sigma_i$ in all $i\in\calI_P^\mathrm{max}$ maximum pooling layers, yielding
\begin{equation}\label{eq:diss_coupling3}
\begin{bmatrix}
        \ddots & \ddots\\
        \ddots & -\widetilde{Q}_{i-2}-R_{i-1} & -\widetilde{S}_{i-1}^\top &&&&\\
        &-\widetilde{S}_{i-1} & -\widetilde{Q}_{i-1}-\Sigma_{i}  &&&&\\
        &&& \Sigma_{i}-\widetilde{R}_{i+1} & -\widetilde{S}_{i+1}^\top &\\
        &&& -\widetilde{S}_{i+1} & \widetilde{Q}_{i+1}-\widetilde{R}_{i+2} & \ddots\\
        &&&& \ddots & \ddots
    \end{bmatrix}\succeq 0.
\end{equation}
We set the resulting coupling matrix \eqref{eq:diss_coupling3} to zero, which results in $Q_{i-1}=-\Sigma_i=-R_{i+1}$ being diagonal matrices. We replace $S_i=0$ and $R_i=-Q_{i-1}$ in \eqref{eq:cert_CNN} and \eqref{eq:cert_FC} to obtain \eqref{eq:cert_reduced_vars_CNN} and the first LMI stated in \eqref{eq:cert_reduced_vars_FC}. Finally, we consider the last linear layer. It is not concatenated with an activation function, such that $\Lambda_{l}=I$ and further, we set $Q_{l}=I$ resulting from $H=0$, which yields the second LMI in \eqref{eq:cert_FC}.
\end{proof}

\end{document}